\documentclass{article}

\PassOptionsToPackage{numbers,compress}{natbib}

\usepackage[preprint]{neurips_2026}

\usepackage{microtype}
\usepackage{graphicx}
\usepackage{subcaption}
\usepackage{booktabs} 
\usepackage{multirow} 
\usepackage{makecell} 
\usepackage[table]{xcolor} 
\definecolor{ourhighlight}{RGB}{230,245,230} 
\newcommand{\hl}{\cellcolor[rgb]{0.92,0.96,0.92}} 
\usepackage{wrapfig}  
\usepackage[breaklinks=true]{hyperref}

\usepackage{amsmath}
\usepackage{amssymb}
\usepackage{amsfonts}
\usepackage{mathtools}
\usepackage{amsthm}
\usepackage{enumitem}
\usepackage{algorithm}
\usepackage{algorithmic}
\usepackage{listings}

\definecolor{codegreen}{rgb}{0.25,0.5,0.35}
\definecolor{codegray}{rgb}{0.5,0.5,0.5}
\definecolor{codepurple}{rgb}{0.58,0,0.82}
\definecolor{backcolour}{rgb}{0.97,0.97,0.97}
\lstdefinestyle{pythonstyle}{
  backgroundcolor=\color{backcolour},
  commentstyle=\color{codegreen},
  keywordstyle=\color{blue},
  numberstyle=\tiny\color{codegray},
  stringstyle=\color{codepurple},
  basicstyle=\ttfamily\small,
  breakatwhitespace=false,
  breaklines=true,
  captionpos=b,
  keepspaces=true,
  numbers=left,
  numbersep=5pt,
  showspaces=false,
  showstringspaces=false,
  showtabs=false,
  tabsize=4,
  language=Python,
  frame=single,
  framerule=0.5pt,
  rulecolor=\color{gray!50},
}

\usepackage[capitalize,noabbrev]{cleveref}

\theoremstyle{plain}
\newtheorem{theorem}{Theorem}[section]
\newtheorem{proposition}[theorem]{Proposition}

\theoremstyle{definition}

\theoremstyle{remark}
\newtheorem{remark}[theorem]{Remark}

\usepackage[textsize=tiny]{todonotes}

\title{VESPO: Variational Sequence-level Soft Policy Optimization for Off-Policy LLM Training}

\author{%
  Guobin Shen
  \quad Chenxiao Zhao
  \quad Xiang Cheng
  \quad Lei Huang
  \quad Xing Yu\thanks{Correspondence to: \texttt{yuanshan2@xiaohongshu.com}, \texttt{floyed\_shen@outlook.com}.} \\
  Xiaohongshu Inc.
}

\begin{document}

\maketitle

\begin{abstract}

  Off-policy updates are inevitable in reinforcement learning (RL) for large language models (LLMs) due to rollout staleness from asynchronous training and mismatches between training and inference engines. Naive importance sampling gives an unbiased correction but suffers from high variance, which is amplified by unbounded ratios and autoregressive generation. Prior remedies either rely on scenario-specific engineering, or trade bias for variance via token-level clipping or sequence-level normalization, yet these approaches remain largely heuristic. We propose \textbf{V}ariational s\textbf{E}quence-level \textbf{S}oft \textbf{P}olicy \textbf{O}ptimization (\textbf{VESPO}). By explicitly incorporating variance reduction into a variational formulation, we derive a principled closed-form reshaping kernel that operates directly on sequence-level importance weights, avoids token-level approximation and length normalization, and admits an explicit variance bound for the deployed kernel. Experiments on math reasoning and code generation show that VESPO maintains stable training under severe off-policy conditions (staleness up to $64\times$) and delivers consistent gains across both dense and Mixture-of-Experts (MoE) models, outperforming recent reshaping baselines under matched setup. Code is available at \url{https://github.com/FloyedShen/VESPO}.

\end{abstract}

\section{Introduction}
\label{sec:intro}
\vspace{-.5em}

\begin{wrapfigure}{r}{0.32\linewidth}
  \centering
  \vspace{-1em}
  \includegraphics[width=\linewidth]{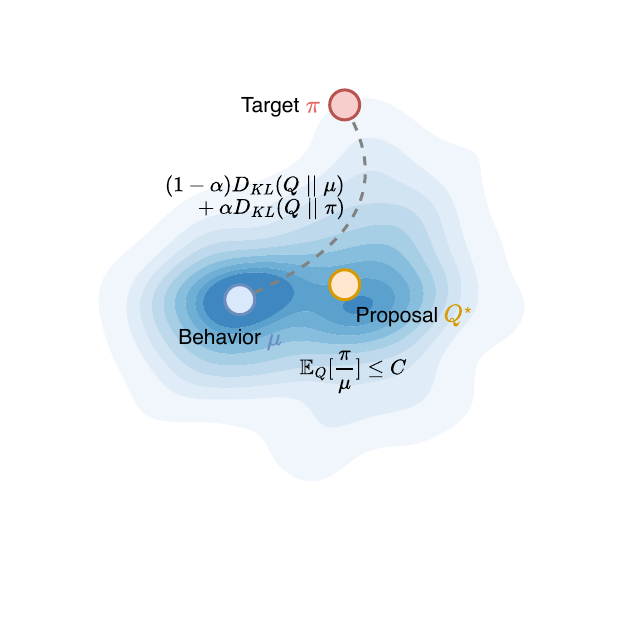}
  \caption{VESPO reformulates IS weight reshaping as finding a proposal $Q^*$ that balances proximity to $\mu$ and $\pi$ under a variance constraint.}
  \label{fig:overview}
  \vspace{-1em}
\end{wrapfigure}
Reinforcement learning (RL) has become a key technique for tackling complex problem-solving tasks with large language models (LLMs), enabling capabilities such as multi-step mathematical reasoning and code generation~\citep{openai2024o1, anthropic2025claude, DeepSeekAI2025DeepSeekR1IR, Yang2025Qwen3TR}.

In practice, off-policy updates arise naturally in RL pipelines for LLMs.
A common source is that systems split large rollout batches into mini-batches for sequential updates~\citep{Yang2025Qwen3TR}, causing later batches to become stale relative to the evolving policy.
Asynchronous systems~\citep{fu2025areal, zhong2025streamrl, noukhovitch2025asynchronous} amplify this by decoupling rollout from training entirely.
Training-inference mismatches introduce further discrepancies, especially in MoE models where routing decisions compound through layers.
To stabilize training under such distribution mismatch, existing works adopt truncated importance sampling (TIS)~\citep{liu-li-2025-rl-collapse}, mask out off-policy samples from training~\citep{hu2025stabilizing}, or replay expert routing for target policy inference~\citep{Zheng2025GroupSP, ma2025stabilizingmoereinforcementlearning}.
While PPO~\citep{Schulman2017ProximalPO} enforces trust region constraints via clipping, it does not directly address the variance challenge of sequence-level importance sampling (IS).

Existing methods address this challenge through various \emph{importance weight transformations}.
Most operate at the token level: GRPO~\citep{shao2024deepseekmath, DeepSeekAI2025DeepSeekR1IR} applies PPO-style clipping to per-token ratios; other methods such as SAPO~\citep{gao2025softadaptivepolicyoptimization} design heuristic transformations for off-policy importance weights~\citep{xi2025bapostabilizingoffpolicyreinforcement, zheng2025prosperity, dwyer2025itsyouitsclipping, roux2025taperedoffpolicyreinforcestable, hu2025reinforcestabilizingcriticfreepolicy}.
However, token-level transformations are a compromise to avoid the multiplicative variance explosion of sequence-level weights, and have been shown to be merely a first-order approximation to their sequence-level counterparts~\citep{Zheng2025GroupSP}.
Sequence-level approaches~\citep{Zheng2025GroupSP, zhao2025geometricmeanpolicyoptimization} introduce length normalization to control variance, but this normalization makes the IS estimator biased; hard clipping is still required on top.
Despite these efforts, principled guidance for designing importance weight transformations remains limited.

We propose \textbf{V}ariational s\textbf{E}quence-level \textbf{S}oft \textbf{P}olicy \textbf{O}ptimization (\textbf{VESPO}), which takes a fundamentally different approach: rather than designing reshaping heuristics, we explicitly incorporate variance reduction for off-policy importance sampling into a variational formulation, yielding a principled closed-form solution (\cref{fig:overview}).
Additionally, we find that the resulting transformation is particularly friendly to sequence-level optimization: unlike previous methods that rely on length normalization to avoid variance explosion, VESPO operates directly on sequence-level importance weights without approximation or normalization.

Our contributions are summarized as follows:
\begin{itemize}[itemsep=-1pt, topsep=-1pt]
  \item We recast importance weight reshaping as a measure change to an implicit proposal distribution, and derive a principled closed-form kernel from a variance-constrained variational objective.

  \item The resulting algorithm, VESPO, operates directly on sequence-level importance weights without length normalization, preserving inter-token dependencies, and admits an explicit variance bound that ensures bounded gradient contributions under arbitrary staleness.

  \item Experiments on math reasoning and code generation demonstrate that VESPO remains stable under staleness up to $64\times$ across dense and MoE architectures, transfers across reward modalities without retuning, and outperforms recent reshaping baselines.
\end{itemize}


\section{Preliminaries}
\label{sec:prelim}
\vspace{-.5em}

\textbf{Notation.}
We consider an autoregressive language model parameterized by $\theta$ as a policy $\pi_\theta$.
Let $x$ denote a query sampled from a dataset $\mathcal{D}$.
In off-policy settings, responses $y = (y_1, \ldots, y_T)$ are sampled from a behavior policy $\mu$ (e.g., an earlier checkpoint or a different inference engine).
The likelihood of generating $y$ given $x$ factorizes as:
\begin{equation}
  \label{eq:autoregressive}
  \pi_\theta(y \mid x) = \prod_{t=1}^{T} \pi_\theta(y_t \mid x, y_{<t}).
\end{equation}
We write $\tau = (x, y)$ for a query-response pair, and $R(\tau) \in \mathbb{R}$ for the sequence-level reward assigned to the complete response.

\textbf{Policy Gradient with Off-Policy Correction.}
The goal is to maximize the expected reward under the current policy:
\begin{equation}
  \label{eq:objective}
  \mathcal{J}(\theta) = \mathbb{E}_{\tau \sim \pi_\theta}\bigl[R(\tau)\bigr].
\end{equation}
When samples are drawn from a behavior policy $\mu$ instead, importance sampling provides an unbiased correction.
Taking the gradient yields the off-policy policy gradient:
\begin{equation}
  \label{eq:off-policy-pg}
  \nabla_\theta \mathcal{J}(\theta) = \mathbb{E}_{\tau \sim \mu}\left[W(\tau) \cdot R(\tau) \cdot \nabla_\theta \log \pi_\theta(\tau)\right],
\end{equation}
where $W(\tau) = \pi_\theta(\tau) / \mu(\tau)$ is the importance weight.
This classical policy gradient view~\citep{sutton1999policy} reveals that the importance weight $W$ serves as a \emph{gradient weighting factor}: it determines how much each sample contributes to the parameter update.
More generally, any modification to the importance weight can be understood as defining a reshaping function $\phi(W)$ that reweights the gradient:
\begin{equation}
  \label{eq:reshaped-pg}
  \nabla_\theta \tilde{\mathcal{J}}(\theta) = \mathbb{E}_{\tau \sim \mu}\left[\phi(W(\tau)) \cdot R(\tau) \cdot \nabla_\theta \log \pi_\theta(\tau)\right].
\end{equation}
This gradient-centric view will be central to our analysis: the practical effect of any weight transformation must ultimately be understood through how it reweights the policy gradient.

\textbf{The Variance Challenge of Sequence-Level IS.}
Expanding the importance weight in terms of token-level ratios reveals a fundamental structural tension.
Define the token-level importance ratio as $\rho_t = \frac{\pi_\theta(y_t \mid x, y_{<t})}{\mu(y_t \mid x, y_{<t})}$.
The sequence-level weight is then a product:
\begin{equation}
  \label{eq:seq-weight}
  W(\tau) = \prod_{t=1}^{T} \rho_t,
\end{equation}
while the log-policy gradient is a sum:
\begin{equation}
  \label{eq:log-grad-sum}
  \nabla_\theta \log \pi_\theta(\tau) = \sum_{t=1}^{T} \nabla_\theta \log \pi_\theta(y_t \mid x, y_{<t}).
\end{equation}
This product-sum structure creates a tension: the gradient contribution of each token is weighted by a global factor $W(\tau)$ that compounds across all $T$ positions.
Even small per-token deviations accumulate multiplicatively, causing $W(\tau)$ to exhibit extreme values for long sequences.
The variance of $W$ grows exponentially with $T$, rendering naive importance sampling impractical.

To tame this variance, existing methods define specific reshaping functions $\phi$ that modify the gradient weighting.
GRPO~\citep{shao2024deepseekmath} operates at the token level with a PPO-style clipped surrogate.
From the gradient perspective, the effective weight function depends on the sign of the advantage $A$:
\begin{equation}
  \label{eq:grpo-phi}
  \phi_{\text{GRPO}}(\rho_t; A) =
  \begin{cases}
    \rho_t, & \text{if } A > 0 \text{ and } \rho_t \leq 1{+}\varepsilon, \\
    \rho_t, & \text{if } A < 0 \text{ and } \rho_t \geq 1{-}\varepsilon, \\
    0,      & \text{otherwise (gradient zeroed)}.
  \end{cases}
\end{equation}
This breaks the product structure and treats each token update independently, yielding only a first-order approximation~\citep{Zheng2025StabilizingRL}.

GSPO~\citep{Zheng2025GroupSP} operates at the sequence level, defining the gradient weight as the geometric mean of token-level ratios (i.e., normalizing by sequence length):
\begin{equation}
  \label{eq:gspo-phi}
  \phi_{\text{GSPO}}(W) = \left(\prod_{t=1}^{T} \rho_t\right)^{\!1/T} = \exp\left(\frac{1}{T}\sum_{t=1}^{T}\log\rho_t\right),
\end{equation}
followed by a clipping mechanism similar to GRPO.
This normalization introduces a length-dependent bias: the implicit proposal distribution varies with $T$, and sequences with identical per-token statistics but different lengths receive identical weights despite having different true importance weights (see \cref{app:length-normalization} for a formal analysis).
These approaches all define $\phi$ heuristically; the question of what constitutes a principled choice of $\phi$ motivates the variational framework we develop next.


\section{VESPO: Variational Sequence-level Soft Policy Optimization}
\label{sec:vespo}
\vspace{-.5em}

We develop a principled framework for designing importance weight transformations.
We first show that any reshaping function $\phi(W)$ implicitly defines a proposal distribution, then formulate the design of $\phi$ as a variational problem with variance constraints, and finally derive a closed-form solution.

\subsection{Weight Reshaping as Measure Change}
\label{sec:measure-change}
\vspace{-.5em}

Standard importance sampling performs an unbiased measure change $\mu \to \pi$, while any transformation $\phi(W)$ induces a different measure change $\mu \to Q$ for some implicit proposal $Q$.
We now formalize this perspective.
For any function $G(\tau)$:
\begin{equation}
  \label{eq:measure-change}
  \mathbb{E}_{\tau \sim \mu}\bigl[\phi(W(\tau)) \cdot G(\tau)\bigr] = Z \cdot \mathbb{E}_{\tau \sim Q}\bigl[G(\tau)\bigr],
\end{equation}
where $Z = \mathbb{E}_{\mu}[\phi(W)]$ is a normalization constant, and $Q$ is defined by
\begin{equation}
  \label{eq:proposal-def}
  Q(\tau) = \frac{1}{Z} \mu(\tau) \cdot \phi(W(\tau)).
\end{equation}
The reshaped gradient (\cref{eq:reshaped-pg}) thus equals $Z \cdot \mathbb{E}_{\tau \sim Q}\bigl[R(\tau)\, \nabla_\theta \log \pi_\theta(\tau)\bigr]$: an expectation under the proposal $Q$ retaining $\pi_\theta$'s score function ($Q$'s own score differs unless $\phi \propto W$).

This is the key insight: any sequence-level reshaping function $\phi(W)$ implicitly defines a proposal distribution $Q$.
This perspective provides a unified lens to analyze existing importance weight transformations (see \cref{app:implicit-proposals} for detailed analysis of specific methods).
Rather than handcrafting $\phi$ directly, we can specify desirable properties of the proposal $Q$ and derive the corresponding $\phi$.
A good proposal should remain close to $\mu$ for sampling efficiency, incorporate $\pi$ to guide optimization, and control variance.
The following subsections formalize these desiderata as a variational objective and derive a closed-form $\phi^*$.

\subsection{Variational Objective}
\label{sec:variational-objective}
\vspace{-.5em}

\textbf{KL prior with importance tilt.}
We seek a proposal $Q$ that retains sample efficiency from the available
$\mu$-samples while shifting mass toward trajectories favored by the target
$\pi$. We formalize this as a KL-regularized linear objective:
\begin{equation}
  \label{eq:vespo-objective}
  \mathcal{J}(Q) \;=\; D_{\mathrm{KL}}(Q \| \mu) \;-\; \alpha\, \mathbb{E}_Q[\log W],
  \qquad \alpha \geq 0.
\end{equation}
The KL term keeps $Q$ close to the sampling distribution, ensuring that
$\mu$-samples remain informative for estimating expectations under $Q$.
The log-importance term rewards $Q$ for placing mass where $\pi$ exceeds
$\mu$ (i.e., $W > 1$), tilting $Q$ toward the target policy and reducing
bias in the gradient estimate. The coefficient $\alpha$ acts as an inverse
temperature controlling the strength of this tilt: $\alpha = 0$ recovers
$Q = \mu$, $\alpha = 1$ yields $Q = \pi$, and $\alpha > 1$ concentrates
$Q$ on increasingly high-importance trajectories. This is the
KL-regularized linear form standard in maximum-entropy RL and
control-as-inference~\citep{levine2018reinforcement}, here applied to the
proposal distribution rather than the policy itself.

\textbf{Variance Constraint.}
Proximity alone is insufficient: finite sample sizes demand variance control.
In importance sampling, the variance of the estimator scales with the second moment $\mathbb{E}_\mu[W^2]$, a classical result that also underlies the effective sample size (ESS) diagnostic.

Under the measure-change view, this second moment can be related to an expectation under $Q$.
By \cref{eq:proposal-def}, we have $Q(\tau) \propto \mu(\tau) \phi(W(\tau))$, so
\begin{equation}
  \label{eq:variance-link}
  \mathbb{E}_{\tau \sim Q}[W(\tau)] \propto \mathbb{E}_{\tau \sim \mu}[\phi(W) \cdot W].
\end{equation}
When $\phi(W) \approx W$ (approaching unbiased IS), this recovers $\mathbb{E}_\mu[W^2]$; for general $\phi$ with $K := \sup_w \phi(w)/w$, the chain $\mathbb{E}_\mu[\phi(W)^2] = Z \cdot \mathbb{E}_Q[\phi(W)] \leq Z K \cdot \mathbb{E}_Q[W]$ (formalized in \cref{prop:variance-bound}) shows that bounding $\mathbb{E}_Q[W]$ controls the second moment for any reshaping kernel. We therefore impose:
\begin{equation}
  \label{eq:moment-constraint}
  \mathbb{E}_{\tau \sim Q}[W(\tau)] \leq C.
\end{equation}

\textbf{Constrained Optimization.}
Combining the KL-regularized tilt objective with the variance constraint:
\begin{align}
  \label{eq:constrained-opt}
  \min_{Q} \;    & D_{\mathrm{KL}}(Q \| \mu) - \alpha\, \mathbb{E}_Q[\log W] \notag    \\
  \text{s.t.} \; & \mathbb{E}_{Q}[W] \leq C, \quad \textstyle\int Q(\tau)\, d\tau = 1.
\end{align}
Introducing Lagrange multipliers $\lambda \geq 0$ for the moment constraint and $\gamma$ for normalization, we obtain the Lagrangian:
\begin{align}
  \label{eq:lagrangian}
  \mathcal{L}(Q, \lambda, \gamma) = \; & D_{\mathrm{KL}}(Q \| \mu) - \alpha\, \mathbb{E}_Q[\log W] \notag                                 \\
                                       & + \lambda \bigl(\mathbb{E}_Q[W] {-} C\bigr) + \gamma \bigl(\textstyle\int Q\, d\tau {-} 1\bigr).
\end{align}

\subsection{Closed-Form Solution}
\label{sec:closed-form}
\vspace{-.5em}

Taking the functional derivative $\frac{\delta \mathcal{L}}{\delta Q} = 0$ yields (see \cref{app:variational-derivation}):
\begin{equation}
  \label{eq:Q-solution}
  Q^*(\tau) \;\propto\; \mu(\tau) \cdot W(\tau)^{\alpha} \cdot \exp(-\lambda W(\tau)).
\end{equation}
Comparing with \cref{eq:proposal-def}, we identify the reshaping function:
\begin{equation}
  \label{eq:gamma-kernel}
  \phi(W) = W^{\alpha} \cdot \exp(-\lambda W).
\end{equation}

This kernel has two components: the power term $W^\alpha$ and the exponential term $\exp(-\lambda W)$ for soft suppression.
It is smooth and differentiable everywhere, avoiding the discontinuities of hard clipping.

\textbf{Surrogate Objective.}
The reshaped gradient implicitly maximizes a smooth surrogate $\mathcal{J}_{\text{VESPO}}(\theta) = \mathbb{E}_{\mu}[f(W(\theta)) A(\tau)]$ that saturates as $W \to \infty$ (\cref{fig:phi-comparison}); $f$ is the lower incomplete gamma function (full derivation in \cref{app:variational-derivation}).

\begin{figure}[t]
  \centering
  \includegraphics[width=1.\linewidth]{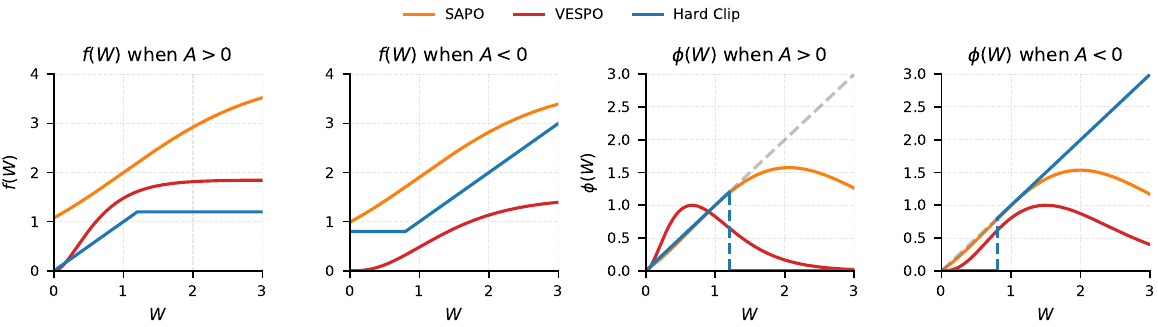}
  \caption{Surrogate objectives $f(W)$ and gradient scaling factors $\phi(W) = W \cdot f'(W)$ for positive and negative advantages. Hard clipping zeros $\phi$ abruptly at the boundary; VESPO peaks near $W{=}1$ and decays smoothly.}
  \label{fig:phi-comparison}
  \vspace{-.5em}
\end{figure}

\textbf{Shifted form for practice.}
In what follows, we use the shifted form $\phi(W) = W^{c_1}\exp(c_2(1{-}W))$ with $c_1 = \alpha$, $c_2 = \lambda$ to ensure $\phi(1) = 1$, so that on-policy samples receive unit weight. This shift differs from \cref{eq:gamma-kernel} only by the multiplicative constant $e^{c_2}$, equivalent to rescaling the learning rate, so the variational derivation carries over unchanged.

\subsection{Variance Bound and Staleness Robustness}
\label{sec:variance-guarantee}
\vspace{-.5em}

The variational derivation in \cref{sec:closed-form} identifies the kernel form, but the constraint $\mathbb{E}_Q[W] \leq C$ controls only the first moment under $Q$. We now establish a formal variance bound that translates this constraint into an explicit guarantee for the second moment under $\mu$, providing a theoretical foundation for the deployed kernel.

\begin{proposition}[Variance Bound]
  \label{prop:variance-bound}
  Let $\phi(W) = W^{c_1}\exp(c_2(1{-}W))$ with $c_1, c_2 > 0$, and define $K = \sup_{w>0} \phi(w)/w$. Under the moment constraint $\mathbb{E}_Q[W] \leq C$,
  \begin{equation}
    \label{eq:variance-bound}
    \mathbb{E}_\mu\bigl[\phi(W)^2\bigr] \;\leq\; Z \cdot K \cdot C, \quad Z = \mathbb{E}_\mu[\phi(W)].
  \end{equation}
  Moreover, $K < \infty$ if and only if $c_1 \geq 1$. When finite, $K$ admits the closed form
  \begin{equation}
    \label{eq:K-formula}
    K = \left(\frac{c_1-1}{c_2}\right)^{c_1-1} \exp(c_2 - c_1 + 1),
  \end{equation}
  attained at $w^* = (c_1{-}1)/c_2$.
\end{proposition}

The proof, given in \cref{app:variance-bound}, follows by setting $g = \phi$ in the measure-change identity (\cref{eq:measure-change}) and applying $\phi(w) \leq K w$. The bound is non-trivial precisely when $c_1 \geq 1$: at our practical settings, $(c_1,c_2)=(2,3)$ gives $K \approx 2.46$ and $(3,2)$ gives $K = 1$, while for $c_1 < 1$ the bound becomes vacuous ($K = \infty$). In the unbiased-IS limit ($\phi \approx W$), the bound reduces to $\mathbb{E}_\mu[W^2] \lesssim Z \cdot K \cdot C$, i.e., the classical second-moment quantity controlling effective sample size (ESS) and IS variance, so \cref{prop:variance-bound} can be viewed as extending this classical control to general $\phi$. The criterion $c_1 \geq 1$ identified by the bound coincides with the inverse-temperature regime of the variational objective (\cref{sec:variational-objective}, $\alpha = c_1 \geq 1$) and is consistent with our practical choices selected purely from empirical performance.

\textbf{Uniform Boundedness.}
The smooth kernel is also uniformly bounded for any $c_1, c_2 > 0$ (proof and explicit $\phi_{\max}$ in \cref{app:variance-bound}). Consequently, no single off-policy sample can produce an unbounded gradient contribution, no matter how stale (i.e., no matter how large its raw weight $W$). Hard-clipping methods achieve a similar property only through discontinuous truncation, while VESPO's smooth attenuation preserves differentiability throughout.

\subsection{The VESPO Algorithm}
\label{sec:vespo-algorithm}
\vspace{-.5em}

We instantiate the shifted kernel from \cref{sec:closed-form} with asymmetric hyperparameters $(c_1, c_2)$ for positive and negative advantages, mirroring the asymmetric clipping in PPO. Recent work~\citep{tang2025rethinkingsamplepolarityreinforcement} has shown that positive and negative samples exhibit different gradient dynamics during training; accordingly, we apply stronger suppression for $A < 0$ with $W < 1$ to prevent excessive penalization of samples the policy already dislikes, as shown in \cref{fig:phi-comparison}.
We treat $(c_1, c_2)$ as tunable hyperparameters, allowing flexibility beyond the specific values implied by the variational derivation. The asymmetric assignment is an instance of the same kernel functional family (rather than a new family), so \cref{prop:variance-bound} applies to each branch since both $(c_1^+, c_2^+)$ and $(c_1^-, c_2^-)$ satisfy $c_1, c_2 > 0$.
Substituting the shifted kernel into \cref{eq:reshaped-pg}, the VESPO gradient estimator becomes:
\begin{equation}
  \label{eq:vespo-gradient}
  \nabla \mathcal{J}_{\text{VESPO}} = \mathbb{E}_{\tau \sim \mu}\bigl[ W^{c_1} \exp(c_2(1 {-} W)) \cdot A(\tau) \cdot \nabla \log \pi_\theta(\tau) \bigr],
\end{equation}
where $A(\tau) = R(\tau) - b$ is the advantage with baseline $b$ (mean reward within each prompt group, following GRPO~\citep{shao2024deepseekmath}). For general $\phi \neq W$, the baseline acts as a control variate for the deployed surrogate $\mathbb{E}_\mu[\phi(W) A(\tau)]$ rather than preserving the unbiasedness identity that holds only for $\phi = W$.

The kernel adapts smoothly to the importance weight: $\phi(W) \approx 1$ near on-policy, the exponential term decays for $W \gg 1$, and the power term down-weights $W \ll 1$. The factorized form gives flexible control: $c_1$ governs $W < 1$ behavior and $c_2$ controls decay for $W > 1$.

\textbf{Computational cost.}
VESPO incurs no additional forward passes or memory beyond GRPO/GSPO: it reuses the same per-token log-probabilities under $\pi_\theta$ and $\mu$, computes $\log W$ as their masked sum, and applies $\phi$ as an elementwise reshaping detached from the computation graph. We evaluate $c_1\log W + c_2(1-W)$ in log-space, exponentiating only at the final step to avoid overflow under extreme staleness. This makes VESPO a drop-in replacement in standard RL pipelines (full pseudocode in \cref{app:algorithm}).

\vspace{-.5em}
\section{Experiments}
\label{sec:experiments}
\vspace{-.5em}

We evaluate VESPO on \textbf{mathematical reasoning} (the primary setting) and \textbf{code generation} (cross-domain transfer; \cref{sec:code-transfer}). For math, we examine two practical sources of off-policy distribution shift: (1) \textbf{policy staleness} from batched rollouts, where later mini-batches are updated using samples from an outdated policy; and (2) \textbf{train-inference mismatch}, where different implementations between training and inference engines produce different outputs for the same input, an effect exacerbated in MoE models due to routing inconsistencies.

\vspace{-.5em}
\subsection{Experimental Setup}
\label{sec:exp-setup}
\vspace{-.5em}

We train on DAPO-Math~\citep{yu2025dapo} (verifier reward~\citep{mathverify2025}) for math reasoning across three model scales: Llama-3.2-3B-Instruct~\citep{grattafiori2024llama3herdmodels}, Qwen3-8B-Base, and Qwen3-30B-A3B-Base~\citep{Yang2025Qwen3TR}, on 32 H20 GPUs with veRL~\citep{sheng2024hybridflow}. Math evaluation uses AIME 2024/2025, AMC 2023, and MATH-500~\citep{hendrycks2021measuring} (avg@$k$). For cross-domain transfer (\cref{sec:code-transfer}), we additionally train Qwen3-30B-A3B-Base on PRIME-RL/Eurus-2-RL-Data~\citep{cui2025process} with execution reward and evaluate on HumanEval+, MBPP+~\citep{liu2023humanevalplus}, and LiveCodeBench v6~\citep{jain2024livecodebench} (pass@10). Best checkpoint per method is selected by average accuracy. Baselines GRPO~\citep{shao2024deepseekmath}, GSPO~\citep{Zheng2025GroupSP}, SAPO~\citep{gao2025softadaptivepolicyoptimization}, TOPR~\citep{roux2025taperedoffpolicyreinforcestable}, CISPO~\citep{chen2025minimax}, BAPO~\citep{xi2025bapostabilizingoffpolicyreinforcement}, etc., use their respective official hyperparameters; VESPO uses $(c_1, c_2) = (2.0, 3.0)$ for $A{>}0$ and $(3.0, 2.0)$ for $A{<}0$, applied identically across all settings (no retuning between math and code). Staleness is simulated by fixing mini-batch size $256$ and varying global batch size; primary experiments use staleness ratio $N{=}\text{gbs}/\text{mbs}{=}8$, ablations span $N \in \{4, 8, 16, 32, 64\}$. Full per-method hyperparameters and infrastructure details are in \cref{app:exp-details}.

\vspace{-.5em}
\subsection{Main Results}
\label{sec:main-results}
\vspace{-.5em}

\begin{wraptable}{r}{0.6\linewidth}
  \vspace{-1.2em}
  \caption{Mathematical reasoning accuracy (\%) with gbs/mbs $= 8$ across three model scales. Best results in \textbf{bold}.}
  \label{tab:main-results}
  \centering
  \footnotesize
  \setlength{\tabcolsep}{2.5pt}
  \begin{tabular}{l|l|cccc|c}
    \toprule
    Model & Method            & AIME25           & AIME24           & AMC23            & MATH500          & Avg              \\
    \midrule
    \multirow{4}{*}{\makecell[l]{Llama-                                                                                      \\3.2-3B-\\Instruct}}
          & GRPO              & 0.5              & 14.5             & 40.6             & \textbf{51.6}    & 26.8             \\
          & GSPO              & 0.2              & \textbf{14.7}    & 43.9             & 47.8             & 26.7             \\
          & SAPO              & \textbf{0.7}     & 12.2             & 34.1             & 51.0             & 24.5             \\
          & \hl\textbf{VESPO} & \hl 0.6          & \hl 13.9         & \hl\textbf{47.3} & \hl 51.3         & \hl\textbf{28.3} \\
    \midrule
    \multirow{4}{*}{\makecell[l]{Qwen3-                                                                                      \\8B-Base}}
          & GRPO              & 27.4             & 40.0             & 74.5             & 76.4             & 54.6             \\
          & GSPO              & 28.8             & 37.7             & 80.8             & 78.4             & 56.4             \\
          & SAPO              & \textbf{36.7}    & 49.0             & 80.0             & 78.0             & \textbf{60.9}    \\
          & \hl\textbf{VESPO} & \hl 33.5         & \hl\textbf{49.4} & \hl\textbf{82.2} & \hl\textbf{78.6} & \hl\textbf{60.9} \\
    \midrule
    \multirow{4}{*}{\makecell[l]{Qwen3-                                                                                      \\30B-A3B-\\Base}}
          & GRPO              & 28.2             & 40.0             & 81.4             & 78.3             & 57.0             \\
          & GSPO              & 25.1             & 43.3             & 83.0             & \textbf{84.8}    & 59.1             \\
          & SAPO              & 21.4             & 27.9             & 73.0             & 84.6             & 51.7             \\
          & \hl\textbf{VESPO} & \hl\textbf{44.3} & \hl\textbf{59.6} & \hl\textbf{91.4} & \hl\textbf{84.8} & \hl\textbf{70.0} \\
    \bottomrule
  \end{tabular}
  \vspace{-1.2em}
\end{wraptable}

\cref{tab:main-results} presents the main results with gbs/mbs = 8 across three model scales.
VESPO achieves the best (or tied-best) average accuracy on all three models, demonstrating the generality of our approach.
Notably, the improvements are most pronounced on Qwen3-30B-A3B-Base, where VESPO outperforms the best baseline by $10.9$ percentage points in average accuracy ($70.0$ vs $59.1$).
This suggests that VESPO's soft suppression of extreme importance weights is particularly beneficial for MoE architectures, where routing inconsistencies amplify distribution shift and make training stability more challenging.
Given these observations, we focus our detailed analysis on Qwen3-30B-A3B-Base in the following sections.

\subsection{Robustness to Policy Staleness}
\label{sec:staleness}
\vspace{-.5em}

We examine robustness to policy staleness by varying the staleness ratio $N = \text{gbs}/\text{mbs}$ from 4 to 64. The training-reward curves in \cref{fig:training-reward-comparison} show VESPO's remarkable consistency: all five curves ($N{=}4$ to $N{=}64$) follow nearly identical trajectories, converging to similar final rewards around $0.7$. In contrast, GRPO saturates early; GSPO degrades as $N$ grows, and additionally suffers a catastrophic collapse around step $1{,}200$ at $N{=}4$; SAPO is competitive at $N{=}4$ but its training reward becomes unstable from $N{=}8$ onward, and downstream accuracy collapses at $N{\geq}32$. These differences translate to downstream performance (\cref{tab:staleness}): VESPO achieves the best average accuracy at every $N$, maintaining $61.8\%$ at $N{=}64$, while GRPO and GSPO degrade to $48.6\%$ and $47.6\%$ respectively, and SAPO collapses entirely.

\begin{table}[t]
  \vspace{-.5em}
  \caption{Effect of staleness ratio $N=\text{gbs/mbs}$ on Qwen3-30B-A3B-Base (avg@$k$, \%). $N{=}8$ is in \cref{tab:main-results}. SAPO's downstream accuracy collapses at $N{\geq}32$.}
  \label{tab:staleness}
  \centering
  \footnotesize
  \setlength{\tabcolsep}{3.2pt}
  \begin{tabular}{l|ccccc|l|ccccc}
    \toprule
    \multicolumn{6}{c|}{$N=4$}                       & \multicolumn{6}{c}{$N=16$}                                                                                                                                                                  \\
    Method                                           & AIME25                     & AIME24        & AMC23         & MATH500       & Avg           & Method         & AIME25        & AIME24        & AMC23         & MATH500       & Avg           \\
    \midrule
    GRPO                                             & 22.1                       & 33.1          & 76.4          & 84.5          & 54.0          & GRPO           & 20.3          & 31.4          & 71.4          & 81.3          & 51.1          \\
    GSPO                                             & 27.6                       & 43.3          & 83.1          & 85.7          & 59.9          & GSPO           & 24.3          & 41.6          & 83.1          & 84.2          & 58.3          \\
    SAPO                                             & 38.4                       & 51.4          & 85.2          & \textbf{85.9} & 65.2          & SAPO           & 19.6          & 26.0          & 72.2          & 84.3          & 50.5          \\
    \rowcolor[rgb]{0.92,0.96,0.92}    \textbf{VESPO} & \textbf{43.1}              & \textbf{60.3} & \textbf{91.2} & 85.4          & \textbf{70.0} & \textbf{VESPO} & \textbf{40.2} & \textbf{53.2} & \textbf{90.8} & \textbf{84.8} & \textbf{67.3} \\
    \midrule
    \multicolumn{6}{c|}{$N=32$}                      & \multicolumn{6}{c}{$N=64$}                                                                                                                                                                  \\
    Method                                           & AIME25                     & AIME24        & AMC23         & MATH500       & Avg           & Method         & AIME25        & AIME24        & AMC23         & MATH500       & Avg           \\
    \midrule
    GRPO                                             & 21.8                       & 33.4          & 73.4          & 80.9          & 52.4          & GRPO           & 14.6          & 28.0          & 69.4          & 82.4          & 48.6          \\
    GSPO                                             & 18.8                       & 27.3          & 73.4          & 80.6          & 50.0          & GSPO           & 15.4          & 24.1          & 73.9          & 76.8          & 47.6          \\
    SAPO                                             & 12.4                       & 7.2           & 29.5          & 34.5          & 20.9          & SAPO           & 3.3           & 7.3           & 23.8          & 31.4          & 16.5          \\
    \rowcolor[rgb]{0.92,0.96,0.92}    \textbf{VESPO} & \textbf{37.7}              & \textbf{51.4} & \textbf{85.2} & \textbf{83.7} & \textbf{64.5} & \textbf{VESPO} & \textbf{34.2} & \textbf{46.2} & \textbf{83.6} & \textbf{83.3} & \textbf{61.8} \\
    \bottomrule
  \end{tabular}
  \vspace{-1em}
\end{table}

\begin{figure}[ht]
  \centering
  \includegraphics[width=\linewidth]{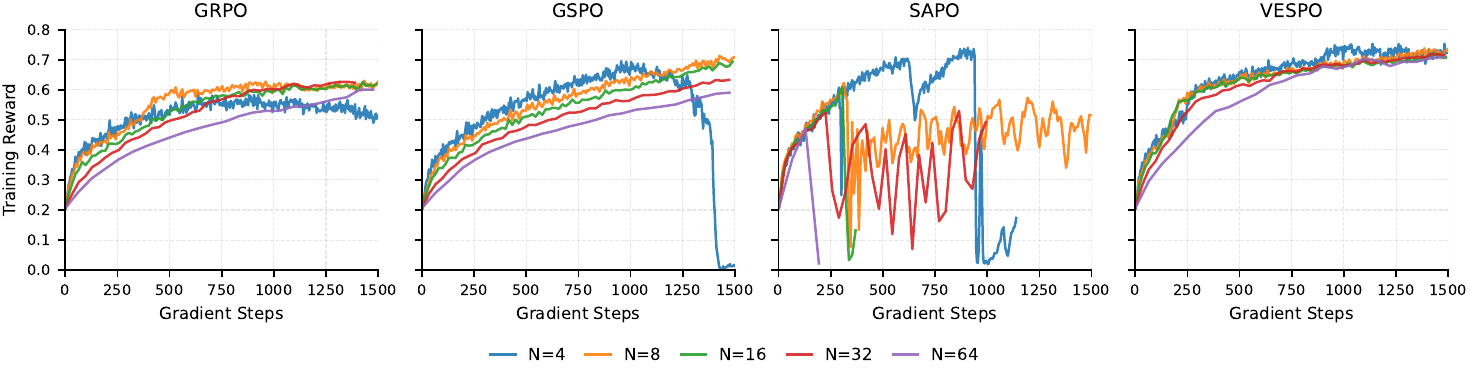}
  \caption{Training reward across staleness levels ($N \in \{4, 8, 16, 32, 64\}$) on Qwen3-30B-A3B-Base. Each panel shows one method; VESPO is the only one with stable, consistent curves across all $N$.}
  \label{fig:training-reward-comparison}
  \vspace{-.5em}
\end{figure}

\begin{figure}[ht]
  \centering
  \includegraphics[width=\linewidth]{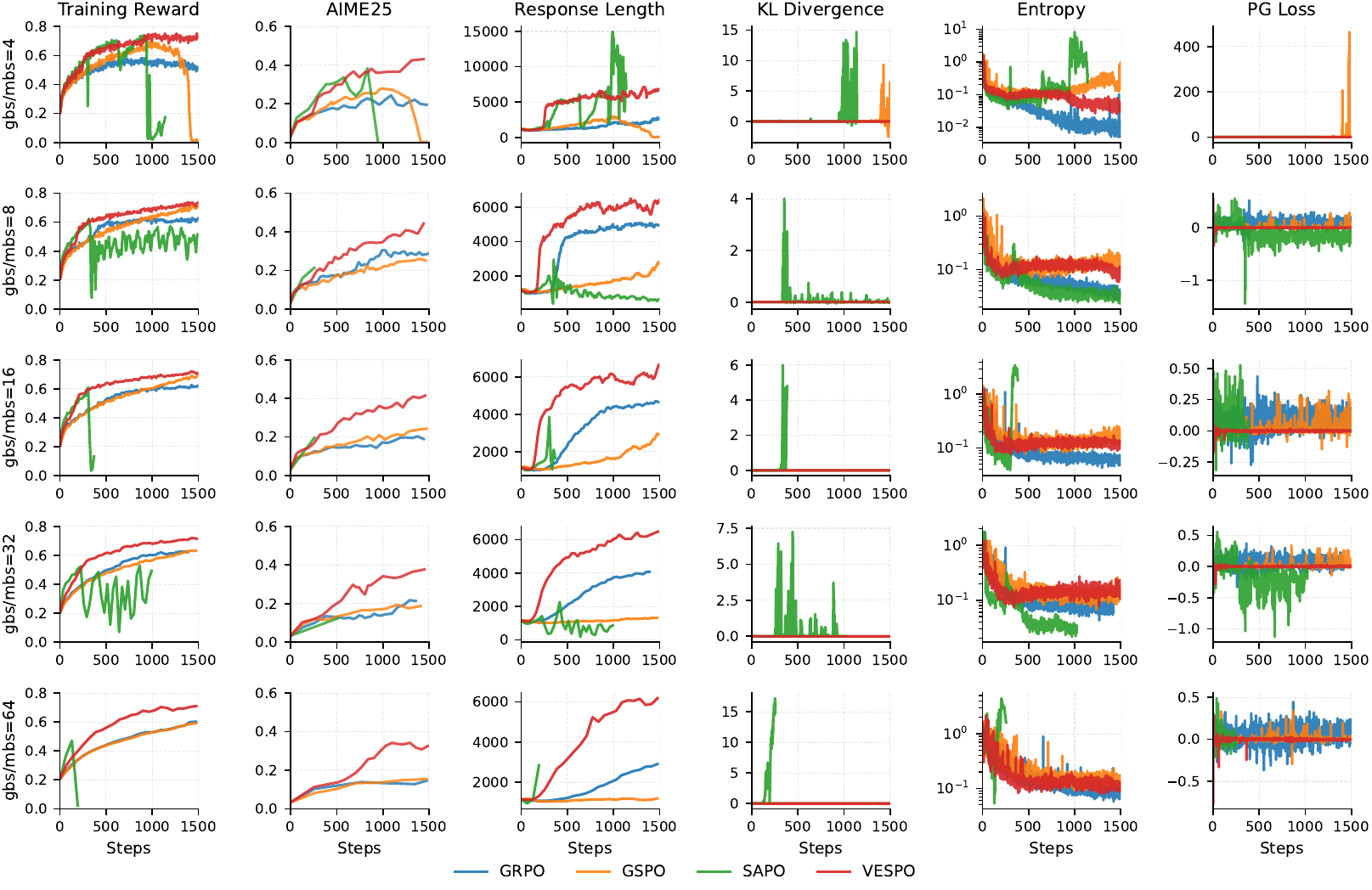}
  \caption{Training dynamics on Qwen3-30B-A3B-Base across staleness $N \in \{4, 8, 16, 32, 64\}$ (rows: $N$; columns: training reward, AIME25, response length, KL, entropy, PG loss). VESPO (red) is the only method with stable training across all $N$; baseline failure modes (GRPO entropy collapse, GSPO length-bias amplification, SAPO under-suppressed $A{<}0$) are detailed in \cref{app:failure-modes}.}
  \label{fig:training-dynamics}
  \vspace{-1em}
\end{figure}

\vspace{-.5em}
\subsection{Comparison with Recent Importance Weight Methods}
\label{sec:recent-baselines}
\vspace{-.5em}

Recent works propose alternative importance weight transformations targeting specific instability sources. We compare VESPO directly against three representative methods on Qwen3-30B-A3B-Base under identical training setup ($N{=}8$):
{TOPR}~\citep{roux2025taperedoffpolicyreinforcestable} hard-clamps sequence-level ratios to $[0,1]$ for negative rewards (and uses weight 1 for positive rewards) with length normalization;
{CISPO}~\citep{chen2025minimax} clips token-level ratios as stop-gradient REINFORCE weights;
{BAPO}~\citep{xi2025bapostabilizingoffpolicyreinforcement} adjusts PPO clipping bounds dynamically with a hard cap for negative advantages.

\begin{table}[ht]
  \vspace{-1.em}
  \begin{minipage}[t]{0.55\linewidth}
    \caption{Comparison with recent importance weight reshaping methods on Qwen3-30B-A3B-Base ($N{=}8$).}
    \label{tab:recent-baselines}
    \centering
    \footnotesize
    \setlength{\tabcolsep}{4pt}
    \begin{tabular}{lccccc}
      \toprule
      Method                                             & AIME25        & AIME24        & AMC23         & MATH500       & Avg           \\
      \midrule
      TOPR                                               & 16.5          & 29.9          & 66.6          & 81.8          & 48.7          \\
      CISPO                                              & 16.0          & 25.5          & 70.2          & 83.0          & 48.7          \\
      BAPO                                               & 34.2          & 47.9          & 88.9          & 83.2          & 63.6          \\
      \midrule
      \rowcolor[rgb]{0.92,0.96,0.92}      \textbf{VESPO} & \textbf{44.3} & \textbf{59.6} & \textbf{91.4} & \textbf{84.8} & \textbf{70.0} \\
      \bottomrule
    \end{tabular}
  \end{minipage}\hfill
  \begin{minipage}[t]{0.43\linewidth}
    \caption{Cross-domain transfer to code generation (pass@10) using the same $(c_1,c_2)$.}
    \label{tab:code-transfer}
    \centering
    \footnotesize
    \setlength{\tabcolsep}{4pt}
    \begin{tabular}{lcccc}
      \toprule
      Method                                             & HE+           & MBPP+         & LCB v6        & Avg           \\
      \midrule
      GRPO                                               & 84.8          & 74.9          & 23.1          & 60.9          \\
      GSPO                                               & 82.9          & 73.0          & 23.1          & 59.7          \\
      SAPO                                               & 86.6          & 74.6          & 24.2          & 61.8          \\
      \midrule
      \rowcolor[rgb]{0.92,0.96,0.92}      \textbf{VESPO} & \textbf{88.4} & \textbf{75.4} & \textbf{25.3} & \textbf{63.0} \\
      \bottomrule
    \end{tabular}
  \end{minipage}
  \vspace{-.5em}
\end{table}

\cref{tab:recent-baselines} shows VESPO substantially outperforms all three baselines: $+10$ on AIME25 over BAPO and $+28$ over TOPR/CISPO. The training dynamics (\cref{fig:recent-baselines} in \cref{app:recent-baselines-dynamics}) reveal distinct failure modes: \emph{CISPO} collapses around step 280 from insufficient sequence-level variance control; \emph{BAPO} exhibits entropy explosion after step 1{,}100 as adaptive clip bounds over-relax; \emph{TOPR} converges slowly with suppressed response length. The advantages of VESPO stem not from any single component, but from the joint effect of sequence-level operation, smooth (non-clipping) kernel, and variational derivation.

\vspace{-.5em}
\subsection{Code Generation}
\label{sec:code-transfer}
\vspace{-.5em}
\cref{tab:code-transfer} reports pass@10 on HumanEval+, MBPP+, and LiveCodeBench v6 after training on code (setup in \cref{sec:exp-setup}; same $(c_1,c_2)$ as in math). VESPO is the best on all three benchmarks, demonstrating that the variational framework yields hyperparameters that generalize across reward modalities (verifier-based math vs.\ execution-based code). Training dynamics and detailed analysis are in \cref{app:code-dynamics}.

\vspace{-.5em}
\subsection{Robustness to Train-Inference Mismatch}
\label{sec:train-infer-mismatch}
\vspace{-.5em}
\begin{wrapfigure}{r}{0.47\linewidth}
  \centering
  \vspace{-1.2em}
  \includegraphics[width=\linewidth]{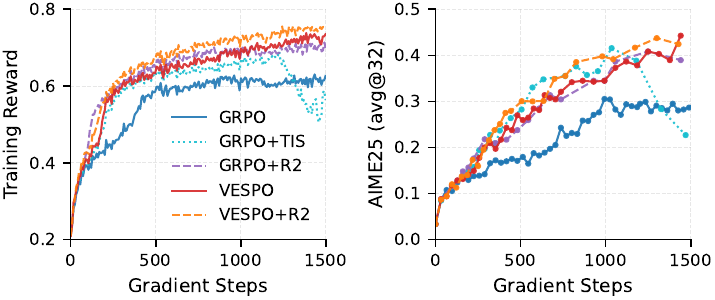}
  \caption{Train-inference mismatch on Qwen3-30B-A3B-Base.}
  \label{fig:train-infer-mismatch}
  \vspace{-1em}
\end{wrapfigure}

A second source of off-policy shift is \emph{train-inference mismatch}: different numerical implementations between training (FSDP/Megatron) and inference (vLLM/SGLang~\citep{zheng2024sglang}) engines produce different outputs, amplified in MoE models. Engineering fixes include Truncated IS (TIS)~\citep{liu-li-2025-rl-collapse} and Routing Replay (R2)~\citep{ma2025stabilizingmoereinforcementlearning}. \cref{fig:train-infer-mismatch} shows vanilla GRPO plateaus at $\sim$$0.60$; adding TIS or R2 improves it. VESPO \emph{without any fixes} matches GRPO+R2, and VESPO+R2 attains the highest reward and best AIME25, showing complementarity with engineering fixes.

  \subsection{Ablations}
  \vspace{-.5em}

  \begin{figure}[ht]
    \centering
    \includegraphics[width=\linewidth]{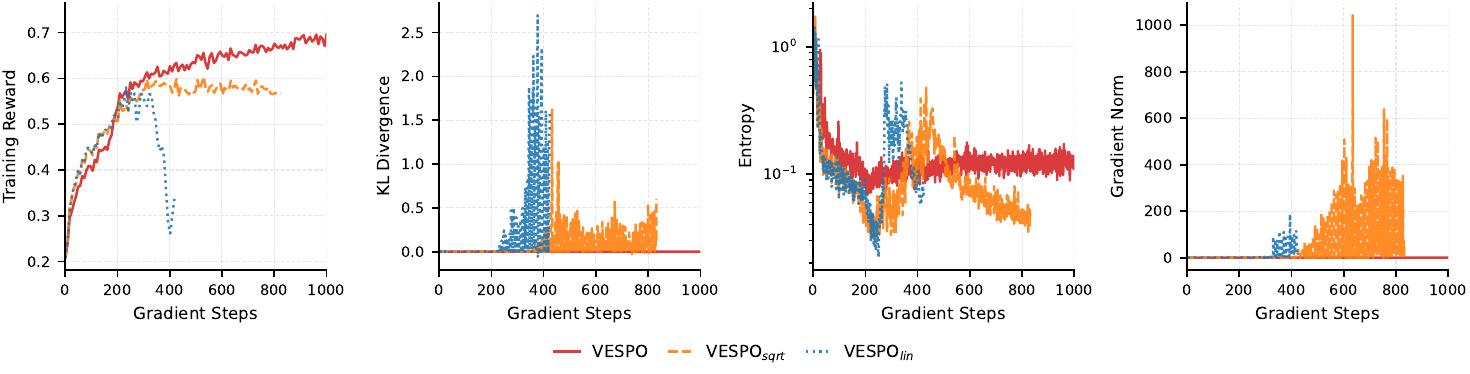}
    \vspace{-1.5em}
    \caption{Ablation on length normalization. VESPO without normalization achieves stable training; adding $\sqrt{T}$ or $T$ normalization causes instability and collapse.}
    \label{fig:ablation-length-norm}
    \vspace{-1.5em}
  \end{figure}

  \textbf{Length normalization.}
  \label{sec:ablation-length-norm}
  A key design choice in VESPO is operating at the sequence level \emph{without} length normalization. Methods like GSPO normalize by $1/T$ to reduce variance, but we hypothesize this creates length-dependent bias. We compare VESPO with two normalized variants: VESPO$_{\text{sqrt}}$ (normalize by $\sqrt{T}$) and VESPO$_{\text{lin}}$ (normalize by $T$). \cref{fig:ablation-length-norm} reveals striking differences: VESPO$_{\text{lin}}$ collapses around step 350 (KL spike + gradient explosion); VESPO$_{\text{sqrt}}$ shows periodic gradient spikes; VESPO without normalization is stable. Length normalization causes longer sequences to dominate batch gradients, biasing toward even longer outputs until collapse, a failure mode VESPO avoids.

  \label{sec:ablation-asymmetric}
  \begin{wrapfigure}{r}{0.3\linewidth}
    \centering
    \vspace{-1.2em}
    \includegraphics[width=\linewidth]{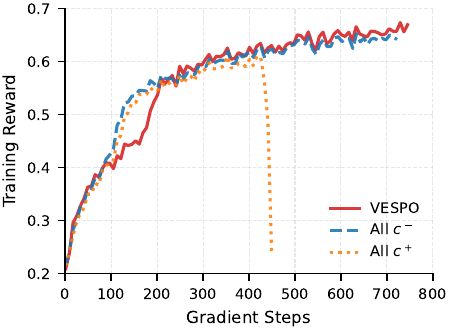}
    \caption{Asymmetric hyperparameter ablation.}
    \label{fig:ablation-asymmetric}
    \vspace{-2em}
  \end{wrapfigure}
  \textbf{Asymmetric hyperparameters.}
  VESPO uses asymmetric hyperparameters $c^+ = (2, 3)$ and $c^- = (3, 2)$ for positive and negative advantages. We compare against symmetric variants (\cref{fig:ablation-asymmetric}). Using $(2, 3)$ for both gives insufficient suppression for $A < 0$, leading to instability; using $(3, 2)$ for both over-suppresses positive samples, slowing learning. The asymmetric design balances these trade-offs. More broadly, VESPO is robust to moderate hyperparameter variations as long as sufficient down-weighting is applied to negative-advantage samples with $W < 1$.


  \vspace{-.5em}
  \section{Related Work}
  \label{sec:related}
  \vspace{-.5em}

  \textbf{Policy Gradient Methods for LLMs.}
  PPO~\citep{Schulman2017ProximalPO} stabilizes updates via a clipped surrogate objective, while value-free alternatives have gained traction for LLM fine-tuning: GRPO~\citep{shao2024deepseekmath} normalizes rewards within sample groups and clips token-level ratios; GSPO~\citep{Zheng2025GroupSP} operates at the sequence level with geometric-mean (length) normalization; DAPO~\citep{yu2025dapo} introduces decoupled clipping and dynamic sampling. These methods control variance via heuristic clipping or normalization. Our measure-change view (\cref{sec:measure-change}) reveals each as a specific choice of implicit proposal distribution, and VESPO derives the weight function from a variational principle rather than manual design.

  \textbf{Importance Weight Reshaping.}
  Recent alternatives to hard clipping can be categorized along three axes: granularity, boundary, and theoretical origin. Token-level hard-boundary methods include CISPO~\citep{chen2025minimax}, BAPO~\citep{xi2025bapostabilizingoffpolicyreinforcement}, GPPO/KLEAR~\citep{su2025klearreasoner}, and M2PO~\citep{zheng2025prosperity}; however, per-token reshaping does not compose into IS under any single sequence-level proposal (\cref{app:implicit-proposals}). SAPO~\citep{gao2025softadaptivepolicyoptimization} uses token-level smooth gating, sharing VESPO's philosophy but inheriting the token-level limitation. TOPR~\citep{roux2025taperedoffpolicyreinforcestable} is sequence-level but with hard clamping to $[0,1]$ for negative rewards and length normalization, suffering from boundary discontinuity and length bias (\cref{app:length-normalization}). {VESPO uniquely combines sequence-level operation, smooth kernel, and variational derivation}, outperforming TOPR/CISPO/BAPO under matched setup (\cref{tab:recent-baselines}). Engineering fixes such as Routing Replay~\citep{ma2025stabilizingmoereinforcementlearning} and truncated IS~\citep{liu-li-2025-rl-collapse} target specific instability sources and are complementary to VESPO (\cref{sec:train-infer-mismatch}).


  \vspace{-.5em}
  \section{Conclusion}
  \label{sec:conclusion}
  \vspace{-.5em}

  We presented a measure-change view of importance weight reshaping and derived VESPO, a sequence-level smooth kernel with explicit variance guarantees. VESPO is stable under staleness up to $64\times$, improves consistently across dense and MoE models, and transfers to code generation without retuning. While the kernel form is principled, specific $(c_1, c_2)$ values still require user choice; our experiments cover verifiable-reward settings (math and code) and models up to 30B-scale, with broader reward modalities and frontier-scale verification left as future directions. Beyond LLM RL, the measure-change formulation and variance-constrained variational design generalize naturally to other variance-bias trade-offs that arise in value learning and reward-model evaluation under distribution shift, suggesting broader applicability across off-policy correction problems.

  \bibliographystyle{plainnat}
  \bibliography{main}

  \newpage
  \appendix

  \begin{center}
    \noindent\rule{\linewidth}{3pt}\\[6pt]
    {\Large\bfseries VESPO: Variational Sequence-level Soft Policy Optimization\\[0.4em] Supplementary Material}\\[6pt]
    \noindent\rule{\linewidth}{1pt}
  \end{center}
  \vspace{1em}

  \noindent\textbf{Table of Contents}
  \vspace{0.3em}
  \hrule
  \vspace{0.5em}
  \noindent
  \hyperref[app:implicit-proposals]{\ref{app:implicit-proposals}}\hspace{1em}Implicit Proposal Distributions of Existing Methods\dotfill\pageref{app:implicit-proposals}\\[0.3em]
  \hyperref[app:variational-derivation]{\ref{app:variational-derivation}}\hspace{1em}Derivation of the Proposal Distribution\dotfill\pageref{app:variational-derivation}\\[0.3em]
  \hyperref[app:variance-bound]{\ref{app:variance-bound}}\hspace{1em}Proof of the Variance Bound and Uniform Boundedness\dotfill\pageref{app:variance-bound}\\[0.3em]
  \hyperref[app:exp-details]{\ref{app:exp-details}}\hspace{1em}Experimental Details\dotfill\pageref{app:exp-details}\\[0.3em]
  \hyperref[app:length-normalization]{\ref{app:length-normalization}}\hspace{1em}Length Normalization Introduces Bias\dotfill\pageref{app:length-normalization}\\[0.3em]
  \hyperref[app:failure-modes]{\ref{app:failure-modes}}\hspace{1em}Additional Analysis of Baseline Failure Modes\dotfill\pageref{app:failure-modes}\\[0.3em]
  \hyperref[app:recent-baselines-dynamics]{\ref{app:recent-baselines-dynamics}}\hspace{1em}Training Dynamics for Recent Baseline Comparison\dotfill\pageref{app:recent-baselines-dynamics}\\[0.3em]
  \hyperref[app:code-dynamics]{\ref{app:code-dynamics}}\hspace{1em}Code Generation Training Dynamics\dotfill\pageref{app:code-dynamics}\\[0.3em]
  \hyperref[app:algorithm]{\ref{app:algorithm}}\hspace{1em}Algorithm Pseudocode\dotfill\pageref{app:algorithm}\\[0.3em]
  \hyperref[app:limitations]{\ref{app:limitations}}\hspace{1em}Limitations and Future Directions\dotfill\pageref{app:limitations}
  \vspace{0.5em}
  \hrule
  \vspace{1.5em}

  \section{Implicit Proposal Distributions of Existing Methods}
  \label{app:implicit-proposals}

  We analyze the implicit proposal distributions induced by existing reshaping strategies under the measure-change framework of \cref{sec:measure-change}.

  \paragraph{Token-level methods (GRPO).}
  GRPO applies clipping independently at each token position, yielding a gradient of the form
  \begin{equation}
    \nabla J_{\text{GRPO}} = \sum_{t=1}^{T} \mathbb{E}_{\tau \sim \mu}\bigl[\phi_t(\rho_t) \cdot A(\tau) \cdot \nabla \log \pi(y_t | x, y_{<t})\bigr],
  \end{equation}
  where $\phi_t(\rho_t)$ is the clipping function applied to the $t$-th token ratio.
  The key observation is that \textbf{different tokens within the same trajectory receive different weights} $\phi_t(\rho_t)$.
  In contrast, sequence-level methods assign a single weight $\phi(W)$ to the entire trajectory:
  \begin{equation}
    \nabla J_{\text{seq}} = \mathbb{E}_{\tau \sim \mu}\bigl[\phi(W) \cdot A(\tau) \cdot \sum_{t=1}^{T} \nabla \log \pi(y_t | x, y_{<t})\bigr].
  \end{equation}
  The token-level formulation cannot be expressed as importance sampling toward any single proposal distribution $Q$, because such a representation requires a uniform weight across all token gradients within each trajectory.
  This token-wise weighting breaks the coherence of sequence-level credit assignment: tokens in the same trajectory that share a common outcome receive inconsistent gradient signals.

  \paragraph{Token-level IS as a First-Order Approximation.}
  We now show that token-level IS is a first-order approximation to sequence-level IS, following \citet{Zheng2025StabilizingRL}.
  Consider the unclipped case where $\phi_t(\rho_t) = \rho_t$ and $\phi(W) = W$.
  The sequence-level gradient (without advantage for clarity) is
  \begin{equation}
    \nabla J_{\text{seq}} = \mathbb{E}_{\tau \sim \mu}\Bigl[W \cdot \sum_{t=1}^{T} \nabla \log \pi(y_t | \cdot)\Bigr] = \mathbb{E}_{\tau \sim \mu}\Bigl[\prod_{s=1}^{T} \rho_s \cdot \sum_{t=1}^{T} \nabla \log \pi(y_t | \cdot)\Bigr].
  \end{equation}
  When $\rho_t \approx 1$ for all $t$ (near on-policy), we can expand $W = \prod_s \rho_s$ via Taylor series:
  \begin{equation}
    W = \prod_{s=1}^{T} \rho_s \approx 1 + \sum_{s=1}^{T} (\rho_s - 1) + \sum_{s < s'} (\rho_s - 1)(\rho_{s'} - 1) + \cdots
  \end{equation}
  Substituting into the gradient and keeping only first-order terms in $(\rho_s - 1)$:
  \begin{align}
    \nabla J_{\text{seq}} & \approx \mathbb{E}_{\tau \sim \mu}\Bigl[\Bigl(1 + \sum_{s=1}^{T} (\rho_s {-} 1)\Bigr) \sum_{t=1}^{T} \nabla \log \pi(y_t | \cdot)\Bigr]                                                                                                                                     \\
                          & = \underbrace{\mathbb{E}_{\tau \sim \mu}\Bigl[\sum_{t=1}^{T} \nabla \log \pi(y_t | \cdot)\Bigr]}_{\text{REINFORCE (no IS)}} + \underbrace{\sum_{s,t} \mathbb{E}_{\tau \sim \mu}\bigl[(\rho_s {-} 1) \nabla \log \pi(y_t | \cdot)\bigr]}_{\text{first-order IS correction}}.
  \end{align}
  The token-level gradient is
  \begin{equation}
    \nabla J_{\text{tok}} = \sum_{t=1}^{T} \mathbb{E}_{\tau \sim \mu}\bigl[\rho_t \cdot \nabla \log \pi(y_t | \cdot)\bigr] = \mathbb{E}_{\tau \sim \mu}\Bigl[\sum_{t=1}^{T} \nabla \log \pi(y_t | \cdot)\Bigr] + \sum_{t=1}^{T} \mathbb{E}_{\tau \sim \mu}\bigl[(\rho_t {-} 1) \nabla \log \pi(y_t | \cdot)\bigr].
  \end{equation}
  Comparing, we see that token-level IS \textbf{only retains the diagonal terms} $(s = t)$ of the first-order correction, discarding cross-token interactions where $s \neq t$.
  The approximation error is
  \begin{equation}
    \label{eq:token-approx-error}
    \nabla J_{\text{seq}} - \nabla J_{\text{tok}} \approx \sum_{s \neq t} \mathbb{E}_{\tau \sim \mu}\bigl[(\rho_s - 1) \nabla \log \pi(y_t | \cdot)\bigr] + O\bigl((\rho - 1)^2\bigr).
  \end{equation}
  This error captures the fact that changing the policy at position $s$ affects the importance of the gradient at position $t$, a cross-token dependency that token-level methods ignore.

  \begin{remark}[When Token-Level Approximation is Reasonable]
    The approximation $\nabla J_{\text{tok}} \approx \nabla J_{\text{seq}}$ is reasonable when:
    \begin{enumerate}[itemsep=0pt, topsep=2pt]
      \item \textbf{Near on-policy}: $\rho_t \approx 1$ for all $t$, so higher-order terms and cross-token terms are small.
      \item \textbf{Short sequences}: Fewer cross-token pairs $(s, t)$ with $s \neq t$ means smaller accumulated error.
    \end{enumerate}
    In practice, RL for LLMs often violates both conditions: off-policy updates are common, and reasoning tasks require long sequences.
    This motivates sequence-level methods that preserve the full product structure of importance weights.
  \end{remark}

  \paragraph{Length normalization (GSPO).}
  GSPO uses the geometric mean $\phi(W) = W^{1/T}$, which induces a proposal distribution that explicitly depends on sequence length:
  \begin{equation}
    Q_{\text{GSPO}}(\tau) \propto \mu(\tau)^{1-1/T} \cdot \pi(\tau)^{1/T}.
  \end{equation}
  As $T \to \infty$, $Q_{\text{GSPO}} \to \mu$, meaning longer sequences receive vanishingly small corrections toward the target policy.
  See \cref{app:length-normalization} for a detailed analysis of the length-dependent bias this introduces.

  \paragraph{Sequence-level hard clipping.}
  If one were to apply hard clipping at the sequence level (truncating $W$ at threshold $c$), the reshaping function would be $\phi(W) = \min(W, c)$, inducing
  \begin{equation}
    Q_{\text{clip}}(\tau) \propto \min\bigl(\pi(\tau), c \cdot \mu(\tau)\bigr).
  \end{equation}
  This is a \emph{truncated distribution}: trajectories with $W > c$ are capped rather than weighted proportionally.
  The discontinuity at $W = c$ can cause optimization difficulties when trajectories cross the boundary during training.
  Note that standard PPO applies clipping at the token level rather than the sequence level, combining the issues of both token-level weighting and hard truncation.


  \section{Derivation of the Proposal Distribution}
  \label{app:variational-derivation}

  We derive the closed-form solution to the constrained optimization problem in \cref{eq:constrained-opt}.

  \begin{proposition}[Solution to the Constrained Problem]
    \label{prop:proposal-solution}
    The solution to
    \begin{align}
      \min_{Q} \;    & D_{\mathrm{KL}}(Q \| \mu) - \alpha\, \mathbb{E}_Q[\log W] \notag   \\
      \text{s.t.} \; & \mathbb{E}_{Q}[W] \leq C, \quad \textstyle\int Q(\tau)\, d\tau = 1
    \end{align}
    is given by
    \begin{equation}
      Q^*(\tau) = \frac{1}{Z}\, \mu(\tau)\, W(\tau)^{\alpha}\, \exp(-\lambda W(\tau)),
    \end{equation}
    where $\lambda \geq 0$ is the Lagrange multiplier for the moment constraint and $Z$ is the normalization constant.
  \end{proposition}

  \begin{proof}
    The Lagrangian, after expanding $D_{\mathrm{KL}}(Q\|\mu) = \int Q\log(Q/\mu)\, d\tau$, is
    \begin{align}
      \mathcal{L}(Q, \lambda, \gamma)
      = \int Q(\tau) \bigl[ & \log Q(\tau) - \log \mu(\tau) - \alpha \log W(\tau) \notag            \\
                            & + \lambda W(\tau) + \gamma \bigr] d\tau \;-\; \gamma \;-\; \lambda C.
    \end{align}
    Taking the functional derivative with respect to $Q$ and setting it to zero:
    \begin{align}
      \frac{\delta \mathcal{L}}{\delta Q} = \log Q(\tau) + 1 - \log \mu(\tau) - \alpha \log W(\tau) + \lambda W(\tau) + \gamma = 0.
    \end{align}
    Solving for $Q$:
    \begin{align}
      \log Q^*(\tau) & = \log \mu(\tau) + \alpha \log W(\tau) - \lambda W(\tau) + \mathrm{const} \\
      Q^*(\tau)      & \propto \mu(\tau)\, W(\tau)^{\alpha}\, \exp(-\lambda W(\tau)).
    \end{align}
    Comparing with $Q(\tau) = \frac{1}{Z} \mu(\tau)\, \phi(W(\tau))$ from \cref{eq:proposal-def}, we identify
    \begin{equation}
      \phi(W) = W^{\alpha} \cdot \exp(-\lambda W). \qedhere
    \end{equation}
  \end{proof}

  \paragraph{Equivalent dual-KL form.}
  Using the identity $D_{\mathrm{KL}}(Q\|\pi) = D_{\mathrm{KL}}(Q\|\mu) - \mathbb{E}_Q[\log W]$, the objective in \cref{eq:vespo-objective} can equivalently be written as
  \begin{equation}
    \label{eq:dual-kl}
    \mathcal{J}(Q) = (1{-}\alpha)\, D_{\mathrm{KL}}(Q \| \mu) + \alpha\, D_{\mathrm{KL}}(Q \| \pi),
  \end{equation}
  which exposes the geometric interpretation: the proposal $Q$ interpolates between $\mu$ ($\alpha{=}0$) and $\pi$ ($\alpha{=}1$). Both forms yield the same $Q^*$; we use the single-KL form in the main text because $D_{\mathrm{KL}}(Q\|\mu)$ retains unit weight in the deployed regime $\alpha \geq 1$.

  \paragraph{Surrogate Objective.}
  The reshaped gradient (\cref{eq:reshaped-pg}) corresponds to maximizing an implicit surrogate. Using $\nabla_\theta W = W \nabla_\theta \log \pi_\theta$, the estimator can be rewritten as
  \begin{equation}
    \mathbb{E}_{\tau \sim \mu}\bigl[\phi(W) \cdot A \cdot \nabla \log \pi_\theta\bigr] = \mathbb{E}_{\tau \sim \mu}\bigl[\tfrac{\phi(W)}{W} \cdot A \cdot \nabla W\bigr].
  \end{equation}
  Defining $f$ by $f'(W) = \phi(W)/W$, this equals $\nabla_\theta \mathbb{E}_{\tau \sim \mu}[f(W)\, A(\tau)]$, since $A(\tau)$ does not depend on $\theta$. Hence VESPO implicitly maximizes
  \begin{equation}
    \mathcal{J}_{\text{VESPO}}(\theta) = \mathbb{E}_{\tau \sim \mu}\bigl[f(W(\theta))\, A(\tau)\bigr].
  \end{equation}
  For $\phi(W) = W^\alpha \exp(-\lambda W)$, $f'(W) = W^{\alpha-1}\exp(-\lambda W)$ integrates to the \emph{lower incomplete gamma function}:
  \begin{equation}
    f(W) = \frac{1}{\lambda^\alpha}\, \gamma(\alpha, \lambda W), \qquad \gamma(a, x) = \int_0^x t^{a-1} e^{-t}\, dt.
  \end{equation}
  This form is smooth and infinitely differentiable, and saturates as $W \to \infty$, providing a principled soft alternative to hard clipping.

  \paragraph{Shifted Form.}
  In practice, we use the shifted form $\phi(W) = W^{c_1} \exp(c_2(1 - W))$, which satisfies $\phi(1) = 1$.
  This ensures that on-policy samples receive unit weight.
  We treat $(c_1, c_2)$ as tunable hyperparameters, allowing flexibility beyond the specific values implied by the variational derivation.


  \section{Proof of the Variance Bound and Uniform Boundedness}
  \label{app:variance-bound}

  We prove \cref{prop:variance-bound} from \cref{sec:variance-guarantee} and the supporting boundedness result for the deployed kernel.

  \begin{proof}[Proof of \cref{prop:variance-bound}]
    By \cref{eq:proposal-def}, $Q(\tau) = \mu(\tau) \phi(W(\tau))/Z$, so for any measurable $g: \mathbb{R}_{>0} \to \mathbb{R}$:
    \begin{equation}
      \label{eq:meas-change-id}
      \mathbb{E}_\mu\bigl[\phi(W) \cdot g(W)\bigr] = Z \cdot \mathbb{E}_Q\bigl[g(W)\bigr].
    \end{equation}
    Setting $g = \phi$ gives $\mathbb{E}_\mu[\phi(W)^2] = Z \cdot \mathbb{E}_Q[\phi(W)]$. Since $\phi(w) = (\phi(w)/w) \cdot w \leq K \cdot w$ for all $w > 0$ where $K = \sup_{w>0} \phi(w)/w$, we have
    \begin{equation}
      \mathbb{E}_Q[\phi(W)] \leq K \cdot \mathbb{E}_Q[W] \leq K \cdot C,
    \end{equation}
    using the moment constraint $\mathbb{E}_Q[W] \leq C$. Combining yields $\mathbb{E}_\mu[\phi(W)^2] \leq Z \cdot K \cdot C$.

    \textbf{Finiteness and closed form of $K$.} We have $\phi(w)/w = w^{c_1-1}\exp(c_2(1-w))$. As $w \to 0^+$, this tends to $\infty$ when $c_1 < 1$ (since $w^{c_1-1} \to \infty$), to $\exp(c_2)$ when $c_1 = 1$, and to $0$ when $c_1 > 1$. As $w \to \infty$, the exponential dominates and the ratio tends to $0$ for any $c_2 > 0$. Hence $K < \infty$ iff $c_1 \geq 1$.

    For $c_1 > 1$, $h(w) := (c_1-1)\log w + c_2(1-w)$ has derivative $h'(w) = (c_1-1)/w - c_2$, vanishing at $w^* = (c_1-1)/c_2 > 0$. The second derivative $-(c_1-1)/w^2 < 0$ confirms an interior maximum. Substituting:
    \begin{equation}
      K = (w^*)^{c_1-1} \exp\!\bigl(c_2(1-w^*)\bigr) = \left(\frac{c_1-1}{c_2}\right)^{c_1-1} \exp(c_2 - c_1 + 1),
    \end{equation}
    yielding the stated formula. For $c_1 = 1$ the supremum $K = \exp(c_2)$ is approached as $w \to 0^+$.
  \end{proof}

  \begin{proposition}[Uniform Boundedness of $\phi$]
    \label{prop:phi-bounded}
    For any $c_1, c_2 > 0$, the kernel $\phi(W) = W^{c_1}\exp(c_2(1-W))$ satisfies
    \begin{equation}
      \sup_{W>0} \phi(W) \;=\; \phi_{\max} \;=\; \left(\frac{c_1}{c_2}\right)^{c_1} \exp(c_2 - c_1),
    \end{equation}
    attained at $W^{**} = c_1/c_2$.
  \end{proposition}

  \begin{proof}
    $\log \phi(W) = c_1 \log W + c_2(1-W)$ has derivative $c_1/W - c_2$, vanishing at $W^{**} = c_1/c_2$. The second derivative $-c_1/W^2 < 0$ confirms a maximum. Substituting yields $\phi_{\max} = (c_1/c_2)^{c_1}\exp(c_2-c_1)$.
  \end{proof}

  \textbf{Numerical values at practical settings.} For $(c_1, c_2) = (2, 3)$: $K = (1/3) e^2 \approx 2.46$ and $\phi_{\max} = (2/3)^2 e \approx 1.21$. For $(c_1, c_2) = (3, 2)$: $K = 1$ and $\phi_{\max} = (3/2)^3 e^{-1} \approx 1.24$. The two propositions together establish that the deployed kernel offers (i) tight second-moment control under the moment constraint and (ii) uniformly bounded gradient contribution from any single off-policy sample, regardless of staleness.


  \section{Experimental Details}
  \label{app:exp-details}

  \paragraph{Training infrastructure.}
  All experiments run on 32 NVIDIA H20 GPUs with veRL~\citep{sheng2024hybridflow}, using vLLM 0.11.0~\citep{kwon2023efficient} for rollout inference, FSDP~\citep{zhao2023pytorch} for the dense models (Llama-3.2-3B-Instruct, Qwen3-8B-Base), and Megatron~\citep{shoeybi2019megatron} for the MoE model (Qwen3-30B-A3B-Base). Each run executes $1{,}500$ gradient steps with maximum sequence length $16{,}384$ tokens and 8 responses per query.

  \paragraph{Reward and evaluation protocol.}
  Math reward is verifier-based via Math-Verify~\citep{mathverify2025}; code reward is execution-based (test case pass/fail). Math evaluation reports avg@$k$ with $k{=}32$ for AIME, $k{=}16$ for AMC, and $k{=}4$ for MATH-500; code evaluation reports pass@10. The best checkpoint per method is selected by average accuracy across the corresponding evaluation suite.

  \paragraph{Baseline hyperparameters.}
  We use the official hyperparameters published with each method:
  \begin{itemize}[itemsep=0pt, topsep=2pt]
    \item \textbf{GRPO}~\citep{shao2024deepseekmath}: token-level PPO clipping with asymmetric bounds $(0.2, 0.28)$ following DAPO~\citep{yu2025dapo}.
    \item \textbf{GSPO}~\citep{Zheng2025GroupSP}: sequence-level with $1/T$ length normalization, clipping bounds $(3\text{e-}4,\, 4\text{e-}4)$.
    \item \textbf{SAPO}~\citep{gao2025softadaptivepolicyoptimization}: token-level soft gating, $\tau_{\text{pos}}{=}1.0$, $\tau_{\text{neg}}{=}1.05$.
  \end{itemize}
  VESPO uses asymmetric $(c_1, c_2) = (2.0, 3.0)$ for $A{>}0$ and $(3.0, 2.0)$ for $A{<}0$, applied identically across all models and training tasks (no retuning between math and code).

  \paragraph{Staleness simulation.}
  We fix mini-batch size $256$ and vary global batch size to obtain staleness ratio $N = \text{gbs}/\text{mbs}$. Each rollout batch is split into $N$ mini-batches for sequential gradient updates, so later mini-batches use parameters increasingly stale relative to the rollout policy. Primary experiments use $N{=}8$; staleness ablations span $N \in \{4, 8, 16, 32, 64\}$.


  \section{Length Normalization Introduces Bias}
  \label{app:length-normalization}

  We analyze how length normalization in sequence-level importance sampling introduces length-dependent bias that conflates distinct trajectories.

  \paragraph{Length-Dependent Bias in GSPO.}
  GSPO uses $\phi(W) = W^{1/T}$, which induces a proposal distribution that explicitly depends on the sequence length $T$:
  \begin{equation}
    Q_{\text{GSPO}}(\tau) \propto \mu(\tau)^{1-1/T} \cdot \pi(\tau)^{1/T}.
  \end{equation}
  As $T \to \infty$, $Q_{\text{GSPO}} \to \mu$: the normalized weight converges to $\exp(\mathbb{E}[\log \rho_t])$, a constant independent of the specific trajectory.
  This causes two fundamental problems:
  \begin{enumerate}[itemsep=2pt, topsep=2pt]
    \item \textbf{Signal dissipation}: All weights collapse toward a constant, losing discriminative power.
    \item \textbf{Conflation of distinct sequences}: Sequences with identical per-token statistics but different lengths receive identical weights, despite having different true importance weights.
  \end{enumerate}

  \begin{proposition}[Conflation under Length Normalization]
    \label{prop:gspo-conflation}
    For any two sequences $\tau_1$, $\tau_2$ with lengths $T_1 \neq T_2$, if $\frac{\log W_1}{T_1} = \frac{\log W_2}{T_2}$, then GSPO assigns identical weights: $\phi_{\text{GSPO}}(W_1) = \phi_{\text{GSPO}}(W_2)$.
    However, their true importance weights differ: $W_1 = e^{T_1 \bar{r}}$ vs $W_2 = e^{T_2 \bar{r}}$ for $\bar{r} = \frac{\log W_1}{T_1} = \frac{\log W_2}{T_2}$.
  \end{proposition}

  This conflation is problematic: a short sequence that is moderately off-policy and a long sequence that is severely off-policy may receive identical gradient weights, even though their contributions to the policy gradient should differ substantially.

  \paragraph{VESPO Avoids Length-Dependent Bias.}
  VESPO uses $\phi(W) = W^{c_1} \exp(c_2(1{-}W))$ without any length normalization.
  The reshaping function depends only on the sequence-level importance weight $W$, not on the sequence length $T$.
  Two sequences with the same average per-token log-ratio $\bar{r}$ but different lengths receive different weights: the longer sequence has $W = e^{T\bar{r}}$, which is appropriately transformed by $\phi$.
  This preserves the discriminative power of the importance weight while the soft-shaping kernel controls variance through exponential suppression of extreme weights.


  \section{Additional Analysis of Baseline Failure Modes}
  \label{app:failure-modes}

  This section provides supplementary visualizations comparing $N{=}4$ vs $N{=}8$ for each baseline (the comprehensive training dynamics figure across all $N$ values is in the main text, \cref{fig:training-dynamics}).

  \begin{figure}[h]
    \centering
    \includegraphics[width=\textwidth]{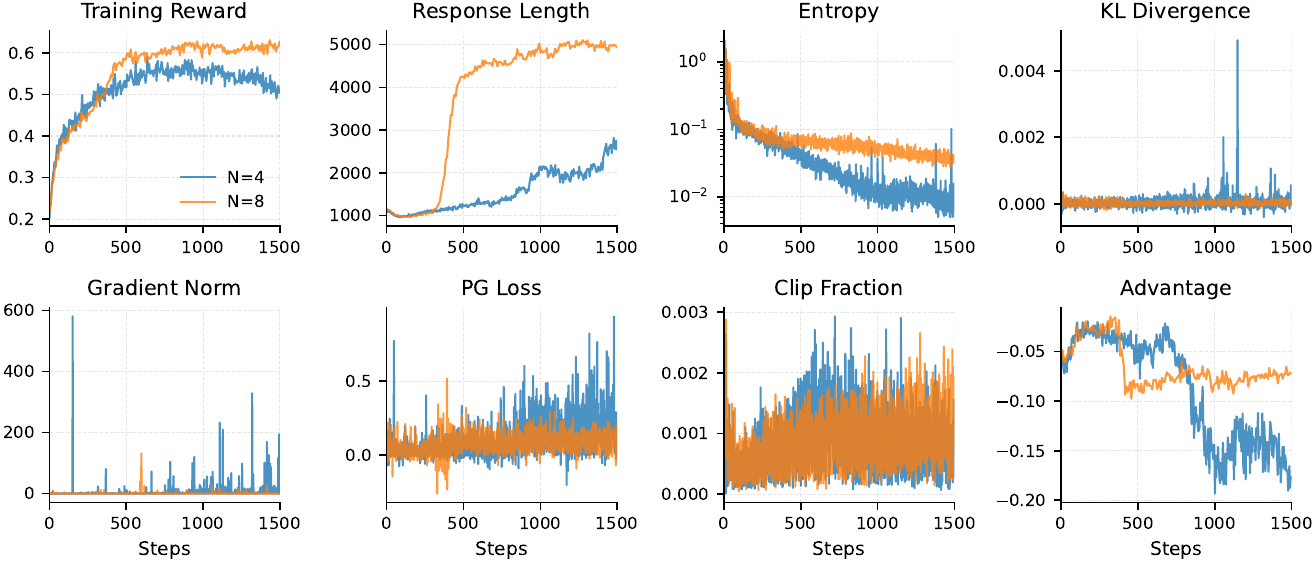}
    \caption{GRPO: $N=4$ (blue) vs $N=8$ (orange). Entropy decreases more rapidly at $N=4$, limiting exploration.}
    \label{fig:grpo-analysis}
  \end{figure}

  \begin{figure}[h]
    \centering
    \includegraphics[width=\textwidth]{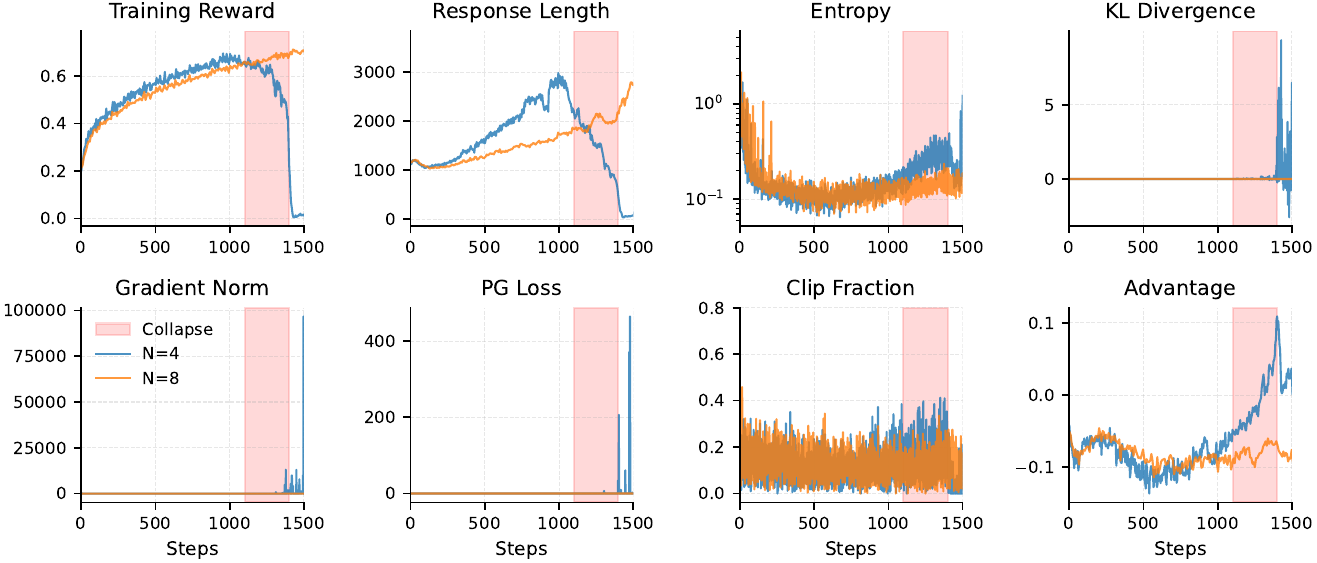}
    \caption{GSPO: $N=4$ (blue) vs $N=8$ (orange). Response length spikes above 3{,}000 tokens at $N=4$ before collapse.}
    \label{fig:gspo-analysis}
  \end{figure}

  \begin{figure}[h]
    \centering
    \includegraphics[width=\textwidth]{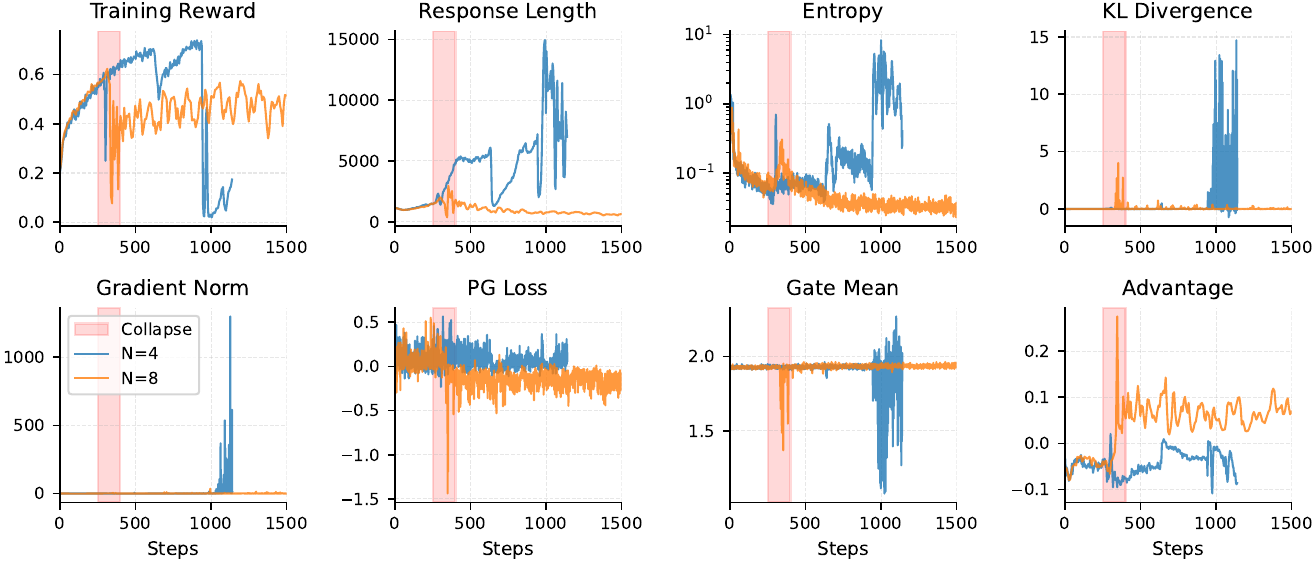}
    \caption{SAPO: $N=4$ (blue) vs $N=8$ (orange). Training reward collapses at $N=8$ due to insufficient suppression for negative advantages.}
    \label{fig:sapo-analysis}
  \end{figure}


  \section{Training Dynamics for Recent Baseline Comparison}
  \label{app:recent-baselines-dynamics}

  \cref{fig:recent-baselines} shows the training dynamics of TOPR, CISPO, BAPO, and VESPO under the matched setup of \cref{sec:recent-baselines}. CISPO collapses around step 280 due to insufficient sequence-level variance control; BAPO exhibits entropy explosion after step 1{,}100 as adaptive clip bounds over-relax; TOPR converges slowly with suppressed response length. VESPO is the only method that maintains stable training across all metrics.

  \begin{figure}[h]
    \centering
    \includegraphics[width=\textwidth]{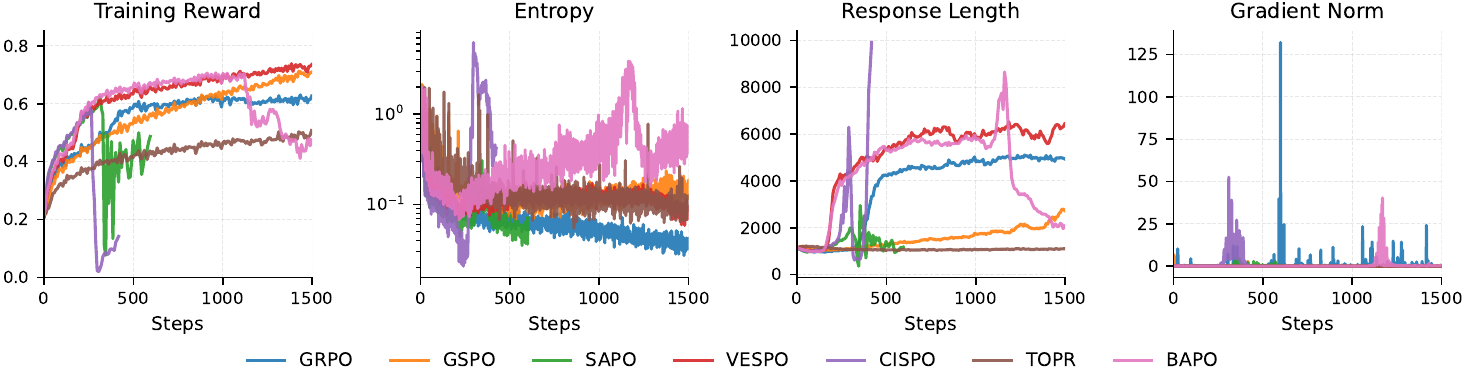}
    \caption{Training dynamics of recent importance weight reshaping methods on Qwen3-30B-A3B-Base ($N{=}8$). CISPO collapses early; BAPO experiences entropy explosion; TOPR converges slowly with suppressed length. VESPO is the only method maintaining stable training across all metrics.}
    \label{fig:recent-baselines}
  \end{figure}


  \section{Code Generation Training Details and Dynamics}
  \label{app:code-dynamics}

  \paragraph{Setup recap.}
  We train Qwen3-30B-A3B-Base on PRIME-RL/Eurus-2-RL-Data~\citep{cui2025process} using the same VESPO hyperparameters $(c_1,c_2)=(2,3)/(3,2)$ as math, without any retuning. The reward is execution-based (test case pass/fail), structurally distinct from the verifier-based math reward. We evaluate on HumanEval+, MBPP+~\citep{liu2023humanevalplus}, and LiveCodeBench v6~\citep{jain2024livecodebench}, reporting pass@10 (\cref{tab:code-transfer}). VESPO is the best on all three benchmarks with zero tuning, demonstrating that the variational framework yields hyperparameters that generalize across reward modalities.

  \paragraph{Training dynamics.}
  \cref{fig:code-dynamics} shows the training dynamics on PRIME-RL/Eurus-2-RL-Data. VESPO maintains the highest training reward with stable entropy and response length throughout 1{,}500 steps. SAPO trains well initially but collapses around step 1{,}250. GRPO and GSPO converge slowly to lower reward levels.

  \begin{figure}[!htbp]
    \centering
    \includegraphics[width=\textwidth]{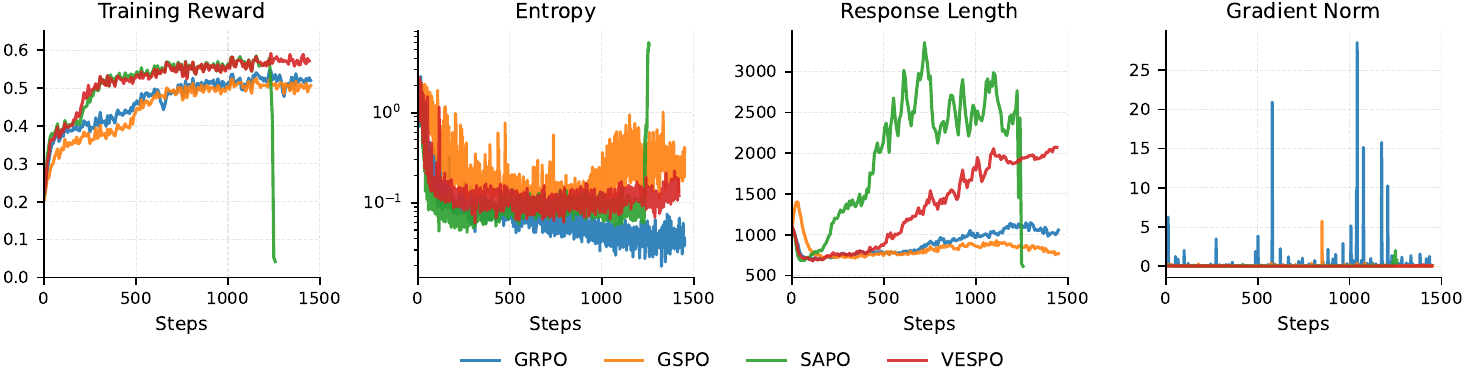}
    \caption{Training dynamics on code generation (PRIME-RL/Eurus-2-RL-Data, Qwen3-30B-A3B-Base, $N{=}8$). VESPO is the only method maintaining stable training; SAPO collapses around step 1{,}250.}
    \label{fig:code-dynamics}
  \end{figure}


  \newpage
  \section{Algorithm Pseudocode}
  \label{app:algorithm}

  \begin{lstlisting}[style=pythonstyle, title={\textbf{VESPO Policy Loss}}]
def compute_policy_loss_vespo(log_pi, log_mu, advantages, mask, c_pos, c_neg):
  """
  Args:
      log_pi: log probs from current policy, shape (batch, seq_len)
      log_mu: log probs from behavior policy, shape (batch, seq_len)
      advantages: sequence-level advantages, shape (batch,)
      mask: response mask, shape (batch, seq_len)
      c_pos: (c1, c2) for positive advantages
      c_neg: (c1, c2) for negative advantages
  """
  # Sequence-level IS ratio in log-space (true product, no length norm)
  log_ratio = log_pi - log_mu
  seq_log_w = (log_ratio * mask).sum(dim=-1)  # (batch,)
  W = exp(seq_log_w)

  # Asymmetric hyperparameter selection
  c1 = where(advantages >= 0, c_pos[0], c_neg[0])
  c2 = where(advantages >= 0, c_pos[1], c_neg[1])

  # VESPO kernel in log-space, normalized so phi(1) = 1
  log_phi = c2 + c1 * seq_log_w - c2 * W
  phi = exp(log_phi).detach()  # gradient scaling only

  # Policy gradient loss: grad = -phi * A * grad_log_pi
  loss = -phi.unsqueeze(-1) * advantages.unsqueeze(-1) * log_pi
  return aggregate(loss, mask)
\end{lstlisting}


  \section{Limitations and Future Directions}
  \label{app:limitations}

  \paragraph{Hyperparameter specification.}
  While the variational framework provides the kernel form $\phi(W) = W^{c_1}\exp(c_2(1-W))$, the specific values $(c_1, c_2)$ require user choice. Our analysis identifies $c_1 \geq 1$ as the principled inverse-temperature regime under which the variance bound (\cref{prop:variance-bound}) is non-vacuous, and we use $(c_1^+, c_2^+) = (2, 3)$ and $(c_1^-, c_2^-) = (3, 2)$ across all settings without retuning between math and code. Adaptive schemes---scheduling $(c_1, c_2)$ along training, automatic per-task search, or learning the hyperparameters from data---could further improve performance and remain an open direction.

  \paragraph{Reward modality scope.}
  Our experiments cover \emph{verifiable-reward} settings: math reasoning with verifier-based reward (Math-Verify) and code generation with execution-based reward (test pass/fail). Effectiveness on preference-based rewards (e.g., RLHF, RLAIF), shaped/dense rewards, or open-ended creative generation remains to be explored. We expect VESPO's smooth variance control to remain beneficial in these regimes since the underlying off-policy IS challenge is the same, but direct empirical validation is needed.

  \paragraph{Model scale.}
  Our largest evaluated model is Qwen3-30B-A3B-Base (30B total parameters with MoE routing). Scaling to substantially larger models ($>$100B) is not directly evaluated due to compute constraints. The observed monotonic trend---larger gains on the MoE Qwen3-30B-A3B-Base than on the dense Qwen3-8B-Base or Llama-3.2-3B-Instruct---suggests favorable scaling behavior, but direct verification at frontier scales remains future work.


\end{document}